\documentclass[11pt]{article}
\usepackage[left=1in, right=1in, top=1in,
bottom=1in]{geometry}

\usepackage{amsmath}
\usepackage{amsthm}
\usepackage{amssymb}
\usepackage{amsfonts}
\usepackage{bbm}
\usepackage{graphicx}
\usepackage{grffile}
\usepackage{wrapfig,epsfig}
\usepackage{url}
\usepackage{epstopdf}
\usepackage{enumitem}
\usepackage{booktabs} 
\usepackage{multirow}
\usepackage{makecell}
\usepackage{subcaption}

\usepackage{algorithm}
\usepackage{algorithmicx}
\usepackage{algpseudocode}

\theoremstyle{plain}
\newtheorem{theorem}{Theorem}[section]
\newtheorem{lem}[theorem]{Lemma}

\newtheorem{assumption}[theorem]{Assumption}

\theoremstyle{definition}

\usepackage[pagebackref,breaklinks,colorlinks,citecolor=blue]{hyperref}

\usepackage{xcolor}
\usepackage{tcolorbox}
\usepackage{enumitem} 
\usepackage{threeparttable}
\usepackage{array}         
\usepackage{colortbl}       

\newtcolorbox{PromptBox}[1][]{
  colback=gray!5,     
  colframe=gray!60,   
  coltitle=black,      
  boxrule=0.5pt,      
  arc=2pt,             
  fonttitle=\bfseries\small, 
  fontupper=\small\ttfamily, 
  title=#1,            
  left=3pt, right=3pt, top=3pt, bottom=3pt, 
  #1
}

\usepackage{wrapfig}
\usepackage{pgfplots}
\pgfplotsset{compat=1.18}

\usepackage{newpxtext}
\usepackage{newpxmath}

\newcommand{\wh}{\widehat}
\newcommand{\wt}{\widetilde}

\newcommand{\R}{\mathbb{R}}

\let\P\relax
\DeclareMathOperator*{\P}{\mathbb{P}}
\DeclareMathOperator*{\E}{{\mathbb{E}}}

\renewcommand{\textsc}[1]{\textnormal{\scshape #1}}

\newcommand{\bx}{\mathbf{x}}

\newcommand{\bz}{\mathbf{z}}

\newcommand{\cD}{\mathcal{D}}

\newcommand{\cJ}{\mathcal{J}}
\newcommand{\cK}{\mathcal{K}}

\newcommand{\cS}{\mathcal{S}}

\makeatletter
\def\blfootnote{\gdef\@thefnmark{}\@footnotetext}
\makeatother

\begin{document}

\title{SLICE: Semantic Latent Injection via Compartmentalized Embedding for Image Watermarking
} 

\author{
  \begin{tabular}{cccc}
Zheng Gao$^1$ ~&~ Yifan Yang$^1$ ~&~ Xiaoyu Li$^1$ ~&~ Xiaoyan Feng$^2$
  \end{tabular}\\[1ex]  
  \begin{tabular}{ccc}
    Haoran Fan$^1$ ~&~ Yang Song$^1$ ~&~ Jiaojiao Jiang$^1$ \\[2ex]  
    \multicolumn{3}{c}{$^1$University of New South Wales} \\[1ex]  
    \multicolumn{3}{c}{$^2$Griffith University} 
  \end{tabular}
}



\maketitle

\begin{abstract}
Watermarking the initial noise of diffusion models has emerged as a promising approach for image provenance, but content-independent noise patterns can be forged via inversion and regeneration attacks. Recent semantic-aware watermarking methods improve robustness by conditioning verification on image semantics. However, their reliance on a single global semantic binding makes them vulnerable to localized but globally coherent semantic edits. 
To address this limitation and provide a trustworthy semantic-aware watermark, we propose \underline{\textbf{S}}emantic \underline{\textbf{L}}atent \underline{\textbf{I}}njection via \underline{\textbf{C}}ompartmentalized \underline{\textbf{E}}mbedding (\textbf{SLICE}). Our framework decouples image semantics into four semantic factors (\textit{subject}, \textit{environment}, \textit{action}, and \textit{detail}) and precisely anchors them to distinct regions in the initial Gaussian noise. This fine-grained semantic binding enables advanced watermark verification where semantic tampering is detectable and localizable. We theoretically justify why SLICE enables robust and reliable tamper localization and provides statistical guarantees on false-accept rates. Experimental results demonstrate that SLICE significantly outperforms existing baselines against advanced semantic-guided regeneration attacks, substantially reducing attack success while preserving image quality and semantic fidelity. Overall, SLICE offers a practical, training-free provenance solution that is both fine-grained in diagnosis and robust to realistic adversarial manipulations.
\blfootnote{Corresponding to: \texttt{zheng.gao1@unsw.edu.au}, \texttt{jiaojiao.jiang@unsw.edu.au}}

\end{abstract}

\section{Introduction}

Diffusion models~\cite{ho2020denoising,song2021score,rombach2022high,esser2024scaling,lai2025principles} have become a dominant paradigm for image synthesis, powering creative tools, editing systems, and large-scale content generation. As these models are increasingly deployed in real applications, determining whether an image was generated by AI and whether it has been subsequently manipulated has become a central provenance challenge~\cite{wang2023survey,cao2025survey}. Watermarking is especially advantageous in this context because it embeds an \emph{active} signal directly into the visual data, ensuring provenance persists even when external metadata is stripped, decoupled, or overwritten~\cite{wen2023treering-article,fernandez2023stable,Yang24GaussianShading-long}.

Existing watermarking methods for generative models span a broad spectrum. Early approaches largely operate in pixel or frequency space, or require modifying model parameters during training or fine-tuning. Representative examples include post-processing methods such as HiNet~\cite{jing2021hinet} and training-based approaches such as Artificial Fingerprinting~\cite{yu2021artificial} and Stable Signature~\cite{fernandez2023stable}. While these methods establish the feasibility of generative provenance, they are not well matched to modern latent diffusion pipelines: post-hoc signals are often brittle under inversion, resampling, or compression, whereas training-dependent schemes introduce substantial computational and operational overhead when models, keys, or user identities must change.

Training-free watermarking in the initial noise space offers a more deployable alternative. Tree-Ring~\cite{wen2023treering-article}, Gaussian Shading~\cite{Yang24GaussianShading-long}, and WIND~\cite{arabi2025hidden} embed watermark patterns directly into the diffusion sampling process and therefore avoid repeated retraining. However, because these methods are largely content-independent, their watermark patterns can often be forged, transplanted, or removed through generative forgery attacks such as Latent Forgery Attack (LFA)~\cite{jain2025forging-article} and Regeneration with a Private Model (RPM)~\cite{arabi2025hidden,muller2025black}.
More fundamentally, recent theoretical and empirical evidence shows that
pixel-level invisible watermarks are provably removable via regeneration
attacks~\cite{zhao2024invisible}, underscoring the need for watermarks
grounded in image semantics rather than low-level signal patterns.
To address this, Semantic Aware Image Watermarking (SEAL)~\cite{Arabi25SEAL-long} takes an important step forward by making the watermark \emph{semantic-aware}: detection is conditioned on the image content rather than on a fixed global pattern alone. This significantly improves resistance to generative forgery attacks and highlights the promise of semantic watermarking for AI provenance.

Despite this progress, current semantic watermarking schemes are still vulnerable to generative forgery attacks because they rely on a \emph{single global semantic binding}.  The recent Coherence-Preserving Semantic Injection (CSI) attack~\cite{gao2026breakingsemanticawarewatermarksllmguided} demonstrates that an adversary can perform locally targeted but globally coherent semantic edits that preserve overall scene plausibility while perturbing exactly those semantics on which watermark verification depends, making the watermark detectable while malicious content is injected. In other words, the image may remain semantically consistent at the holistic level and the watermark remains valid, yet critical local factors---such as the \textit{subject}, \textit{action}, \textit{environment}, or a discriminative \textit{detail}---can be maliciously altered. Once the watermark is tied to a monolithic global descriptor, such fine-grained edits can invalidate the binding without visibly destroying the image.

Our key observation is that image semantics is not monolithic. For provenance, an image is better modeled as a composition of partially independent semantic factors that are also spatially grounded. Recent advances in diffusion control and latent decomposition suggest that semantics can be localized and manipulated in a structured way without retraining, through spatial control and semantically interpretable latent directions~\cite{xie2023boxdiff,bartal2023multidiffusion,gandikota2025sliderspace}. This motivates a different watermarking principle: rather than binding one global semantic code to the entire latent, we should \emph{disentangle semantics and compartmentalize the watermark}. Under such a design, localized semantic edits should induce localized verification failures, making tampering not only detectable but also localizable.

Motivated by the above principle, we propose \underline{\textbf{S}}emantic \underline{\textbf{L}}atent \underline{\textbf{I}}njection via \underline{\textbf{C}}ompartmentalized \underline{\textbf{E}}mbedding (\textbf{SLICE}), a training-free semantic watermarking framework for diffusion models. Given a reference image, SLICE extracts a structured semantic representation and decomposes it into four semantic factors: \emph{subject}, \emph{environment}, \emph{action}, and \emph{detail}. Using keyed hashing and noise synthesis, SLICE then binds each factor to a dedicated, non-overlapping partition of the initial latent noise. This spatial-semantic design yields a compartmentalized watermark whose evidence is distributed across multiple latent partitions rather than collapsed into a single global code. At detection time, SLICE reconstructs reference noise from the semantics of a suspect image and performs both partition-wise and global verification. This enables the detector to distinguish among pristine watermarked images, images containing localized semantic tampering, and images that are unwatermarked or severely altered. By aligning semantic factors with separate latent partitions, SLICE turns localized edits into partition-specific mismatches, addressing the limitation of existing semantic schemes that are insensitive to such edits.   

We summarize our contributions as follows: (i) We identify a key limitation of existing semantic watermarks: global semantic binding is vulnerable to localized but globally coherent edits. (ii) We propose SLICE, a training-free watermarking framework that decomposes image semantics into four factors and binds each to a dedicated latent partition via compartmentalized embedding, enabling a unified verification strategy that authenticates global provenance while localizing semantic tampering. (iii) We provide a theoretical justification for SLICE’s robust and reliable tamper localization, along with statistical guarantees on false-accept rates. (iv) Extensive experiments are conducted to demonstrate strong robustness of SLICE against generative forgery attacks while preserving image quality and semantic fidelity.

\paragraph{Roadmap.} Section~\ref{sec:rela} reviews some related work. Section~\ref{sec:method} details the SLICE framework, covering factorized semantic extraction, spatially-partitioned injection, and three-state detection. Section~\ref{sec:theory} provides theoretical guarantees on tamper localization and false-accept rates. Section~\ref{sec:analysis} justifies our prompt language selection. Section~\ref{sec:experi} presents experimental results. Section~\ref{sec:conclusion} concludes the paper and outlines future directions.

\section{Related Work}\label{sec:rela}

\subsection{Watermarking for Generative Models}

Watermarking for generative models has evolved along parallel tracks in both
language and image domains. In LLMs, early inference-time methods established
detection through token-level distribution
biasing~\cite{kirchenbauer2023watermark}, with subsequent work introducing
unbiased or distortion-free
schemes~\cite{hu2023unbiased,kuditipudi2023robust}. A key shift occurred when
researchers began anchoring watermark evidence to higher-level
structure---SemStamp~\cite{hou2024semstamp} binds to sentence-level semantics,
while multi-bit methods~\cite{yoo2024advancing,feng2025bimark} further enable
source tracing with provable quality preservation. This trajectory, from global
statistical signals toward semantically grounded provenance, is mirrored in the
image domain.

For diffusion models, existing watermarking techniques primarily fall into three
paradigms. Post-processing methods are highly vulnerable to regenerative attacks
(e.g., SDEdit or JPEG compression). Model fine-tuning methods (e.g., Stable
Signature~\cite{fernandez2023stable}, WOUAF~\cite{kim2024wouaf}) embed signals
into network weights to improve robustness, but incur prohibitive training costs
and degrade zero-shot generation quality. Our approach aligns with the
training-free paradigm~\cite{wen2023treering-article,gesny2026guidance}, which
modifies the initial noise or inference trajectory to preserve generative
fidelity. Representative methods include Tree-Ring~\cite{wen2023treering-article},
Gaussian Shading~\cite{Yang24GaussianShading-long}, WIND~\cite{arabi2025hidden},
and ROBIN~\cite{huang2024robin}, Ring-ID~\cite{ci2024ringid} and PRC~\cite{gunn2025detectable}. However, these mainstream
training-free methods rely on global, content-independent spatial injection,
leaving them susceptible to localized geometric deformations or cropping.
This content-independence also exposes them to \emph{generative forgery attacks}.
Black-box adversaries
can forge or remove latent-noise watermarks by leveraging a proxy diffusion
model~\cite{muller2025black}, while a recent theoretical analysis frames semantic
watermark forgery as a rate--distortion problem~\cite{lee2025on}, revealing fundamental limits on
proxy-based attacks. The Coherence-Preserving Semantic Injection (CSI)
attack~\cite{gao2026breakingsemanticawarewatermarksllmguided} further shows that locally targeted but globally
coherent edits can bypass semantic-aware watermarks such as
SEAL~\cite{Arabi25SEAL-long}. These attacks collectively motivate the need for
fine-grained semantic binding, which is the core design principle of SLICE.

\subsection{Diffusion Models and Latent Inversion}

Latent Diffusion Models
(LDMs)~\cite{ho2020denoising,song2021score,rombach2022high} have revolutionized
generative AI by performing iterative denoising within a compressed latent space.
For watermarking, their deterministic sampling mechanism (e.g., DDIM) is crucial,
as it establishes a strict bijective mapping between the initial noise
distribution and the generated
image~\cite{song2020ddim,song2021score}. Recent analyses on inversion
stability~\cite{staniszewski2025there,lin2024schedule} confirm that optimizing
the noise schedule effectively mitigates accumulated trajectory errors.
While more recent architectures based on rectified
flows~\cite{esser2024scaling} adopt different sampling strategies, the
principle of deterministic invertibility remains central.
This structural stability theoretically renders the initial noise space an
optimal, mathematically rigorous carrier for covert signal injection.

\subsection{Semantic Control and Disentangled Representations}

To overcome the limitation of global watermarks, we draw upon recent
advancements in training-free semantic control. In LDMs, self-attention dictates
geometric topology, while cross-attention precisely localizes text-driven
features~\cite{liu2024towards}. Frameworks like
Attend-and-Excite~\cite{chefer2023attend} and BoxDiff~\cite{xie2023boxdiff}
exploit these attention maps during inference to achieve strict spatial
partitioning. Concurrently, latent disentanglement
research~\cite{gandikota2025sliderspace,yang2024diffusion,bartal2023multidiffusion}
demonstrates the ability to isolate specific visual attributes without semantic
interference. We repurpose these mechanisms to strictly anchor the watermark
within specific, disentangled semantic regions (e.g., a foreground object's
attention mask). This semantic-binding strategy ensures that as long as the core
object remains visually recognizable after an attack, the embedded watermark can
be successfully recovered.
\section{Method}\label{sec:method}

\begin{figure*}[t]
    \centering
    \includegraphics[width=0.95\linewidth]{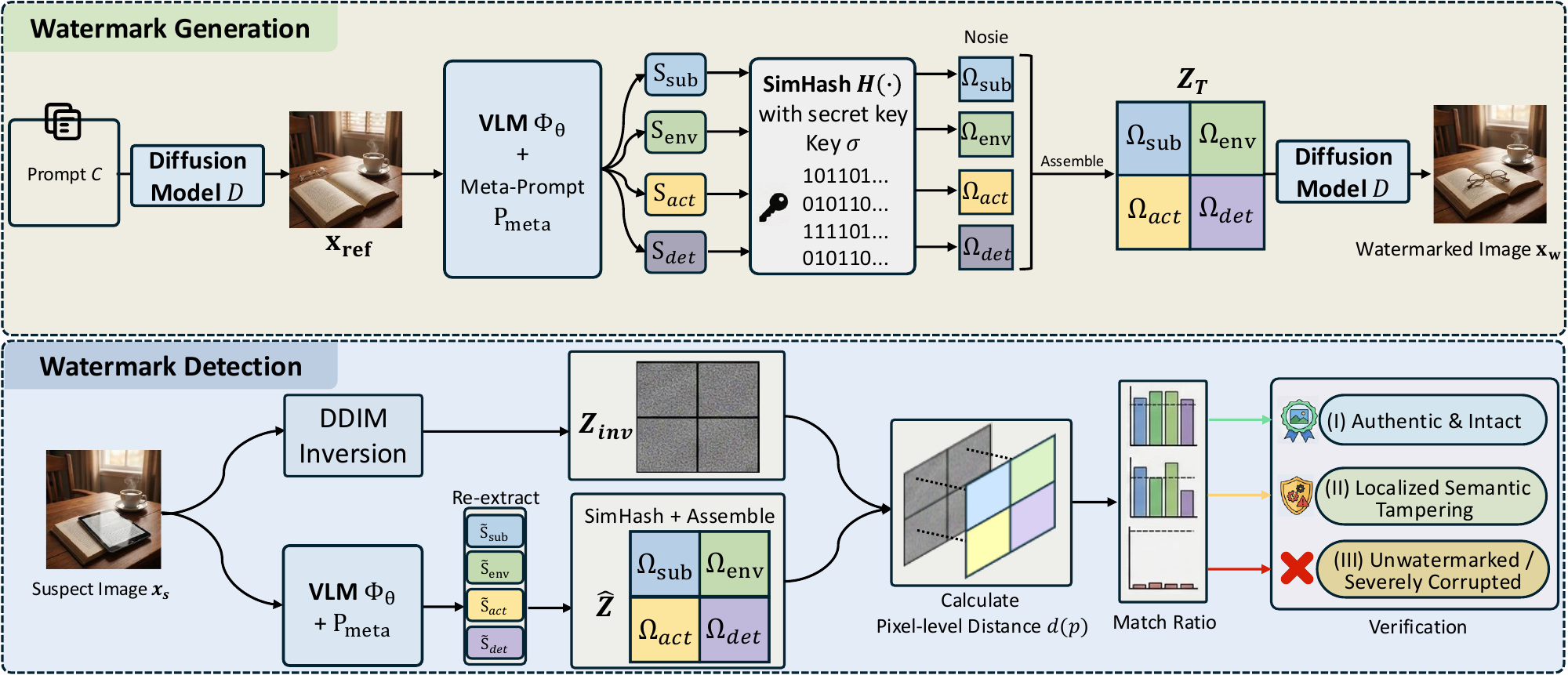} 
    \caption{\textbf{The overall framework of SLICE.} 
    }
    \label{fig:framework}
\end{figure*}

In this section, we introduce the framework of SLICE, as illustrated in Fig.~\ref{fig:framework}. At high level, SLICE embeds a semantic watermark by decomposing image content into four semantic factors (\textit{subject}, \textit{environment}, \textit{action}, and \textit{detail}) and binding each factor to a dedicated spatial partition of the initial latent. At detection time, the same factorized semantics are re-extracted from a suspect image to reconstruct a reference latent, which is then compared with the inverted latent in a partition-wise manner. This design enables both global provenance verification and localization of semantic tampering.

\subsection{Prompt-Conditioned Factorized Semantic Extraction}

Given a conditioning prompt $c$, we first sample a random latent $\bz_\mathrm{rand} \in \R^{h \times w \times d}$ from a Gaussian distribution, and generate an unwatermarked reference image $\bx_\mathrm{ref} = \cD(\bz_\mathrm{rand}, c)\in \R^{H \times W \times 3}$, where $\cD$ denotes a pretrained latent diffusion model. The reference image  $\bx_\mathrm{ref}$ represents the intended visual content and serves as the basis for semantic watermark construction. 

To make the watermark sensitive to localized semantic edits, we avoid using a single holistic caption and instead decompose the image into four semantic factors. Specifically, we define the set
\begin{align*}
    \cK = \{\mathrm{sub}, \mathrm{env}, \mathrm{act}, \mathrm{det}\}
\end{align*}
whose elements correspond to \textit{subject}, \textit{environment}, \textit{action}, and \textit{detail}, respectively. These four factors capture crucial elements of the semantics of an image from distinct and complete aspects. To obtain these factors, we then use a vision-language model (VLM) $\Phi_\theta$ together with a carefully designed meta-prompt $\mathcal{P}_{\mathrm{meta}}$ to extract a structured descriptor set
\begin{align*}
    \cS = \{s_k\}_{k \in \cK} = \Phi_\theta(\bx_\mathrm{ref}, \mathcal{P}_{\mathrm{meta}}).
\end{align*}

The role of $\mathcal{P}_{\mathrm{meta}}$ is to constrain the VLM output into a consistent four-field format, so that each descriptor captures one semantic factor rather than a mixed description of the entire image. This factorized representation is crucial for SLICE: it provides a stable semantic basis for watermark generation and later enables factor-specific verification during detection. The details of the meta-prompt are in Fig.~\ref{fig:meta_prompt}.

\begin{figure}[h]
    \centering
    \small
    \begin{PromptBox}[title={Meta-Prompt $\mathcal{P}_{\mathrm{meta}}$ for Factorized Semantic Extraction}]
\textbf{[System Instruction]}
You are a semantic disentanglement engine. Below is a demonstration of how to decouple visual semantics into four semantic factors. \textbf{Based on this example, strictly analyze the target image and output the concise descriptors for the following four categories:}

\textbf{[Demonstration]}
\textit{Input:} $<$Image: A young boy running on a grassy field$>$
\textit{Output:}
\begin{itemize}[nosep, leftmargin=1em]
    \item \textbf{[Subject]}: Young boy.
    \item \textbf{[Action]}: Running, sprinting forward.
    \item \textbf{[Environment]}: Grassy field, daytime, park setting.
    \item \textbf{[Detail]}: Red t-shirt, motion blur on legs.
\end{itemize}

\textbf{[Inference Task]}
\textit{Input:} $<$Target Image $\mathbf{x}_{ref}$$>$
\textit{Output:}
(Strictly follow the format above to generate the four descriptors...)
    \end{PromptBox}
    
\caption{\textbf{Structure of the Meta-Prompt $\mathcal{P}_{\mathrm{meta}}$.} 
}
    \label{fig:meta_prompt}
\end{figure}

\subsection{Spatially-Partitioned Semantic Injection}

Let the latent space of the diffusion model be a subset of $\R^{h \times w \times d}$.
To bind each semantic factor to a distinct spatial region, we partition the latent grid into four disjoint regions $\{\Omega_k\}_{k \in \cK}$ such that
\begin{align*}
    \Omega_k \cap \Omega_{k'} = \varnothing ~~\mathrm{for}~ k \neq k', \quad \bigcup_{k \in \cK} \Omega_k = \{1,\ldots, h\} \times \{1, \ldots, w\},
\end{align*}
where each region $\Omega_k$ is assigned to one semantic factor $k \in \cK$.

Let $E(\cdot)$ denote the text encoder and let $H(\cdot,\cdot,\sigma)$ denote a keyed hash-based synthesis function parameterized by a secret key $\sigma$. For each semantic factor $k \in \cK$, and each spatial position $p = (i,j) \in \Omega_k$, we generate the corresponding latent value as
\begin{align*}
    \bz_T(p) = H(E(s_k), p, \sigma), \quad  \forall p \in \Omega_k, k \in \cK.
\end{align*}
In other words, the latent entries inside each region $\Omega_k$ depend only on the descriptor $s_k$, the spatial position $p$, and the secret key $\sigma$.

By assembling all regions, we obtain the full watermarked latent $\bz_T$, and the final watermarked image is generated as $\bx_w = \cD(\bz_T, c)$. This compartmentalized design is the core of SLICE. Unlike global semantic watermarking methods that embed a single signal across the entire latent, our method binds each semantic factor to a dedicated latent partition. As a result, a localized semantic edit is expected to disrupt mainly the corresponding partition, while unrelated partitions remain comparatively stable.

\subsection{Watermark Detection}

At detection time, given a suspect image $x_s$, we first apply the deterministic DDIM inversion~\cite{song2020ddim} to recover its approximate initial latent $\bz_\mathrm{inv} = \mathrm{Inversion}(\bx_s)$. Next, using the same VLM $\Phi_\theta$ and the same meta prompt $\mathcal{P}_{\mathrm{meta}}$, we re-extract the semantic descriptors from the suspect image:
\begin{align*}
    \wt{S} = \{\wt{s}\}_{k \in \cK} = \Phi_\theta(x_s, \mathcal{P}_{\mathrm{meta}}).
\end{align*}
Based on these re-extracted descriptors, we reconstruct a reference latent $\wh \bz_T$ using the same synthesis rule as in the generation phase:
\begin{align*}
    \wh \bz_T(p) = H(E(\wt{s}_k), p, \sigma)), \quad p \in \Omega_k.
\end{align*}

We then compare the inverted latent $\bz_\mathrm{inv}$ and the reconstructed latent $\wh{\bz}_T$ in a pointwise manner. For each spatial location $p$, we define the distance 
$
    d(p) = \|\bz_\mathrm{inv} - \wh{\bz}_T(p)\|_2,
$
where the norm is computed over the latent dimension. A location is considered matched if $d(p) \leq \tau$ where $\tau$ is a pixel-level distance threshold. We define the region-level match ratio and the global match ratio as
\begin{align*}
    m_k = \frac{1}{|\Omega_k|} \sum_{p \in \Omega_k} \mathbf{1}[d(p) \leq \tau_k], ~~ \forall k \in \cK, ~~~ \mathrm{and}~~~ m_{g} = \frac{1}{hw} \sum_{p} \mathbf{1}[d(p) \leq \tau_d],
\end{align*}
respectively, where $\mathbf{1}[\cdot]$ denotes the indicator function.

Using a set of local thresholds $\{\tau_k\}_{k \in \cK}$ and a global threshold $\tau_g$, the detector outputs one of three states:

\textbf{State I: Authentic and Intact.} The image is classified as authentic and intact when $m_g \geq \tau_g$ and $m_k \geq \tau_k$ for all $k \in \cK$. This indicates that both the global watermark and all factor-specific partitions remain consistent.

\textbf{State II: Localized Semantic Tampering.} The image is classified as semantically tampered when $m_g \geq \tau_g$ but $m_k < \tau_k$ for some $k \in \cK$. In this case, the image still matches the protected model at a global level, but at least one semantic factor fails local verification. The failed partition directly indicates the likely tampered factor.

\textbf{State III: Unwatermarked or Severely Corrupted.} The image is classified as unwatermarked or severely corrupted when $m_g < \tau_g$. This means that the watermark is either absent or has been destroyed by a strong transformation.
\section{Theoretical Analysis}
\label{sec:theory}
In this section, we present a theoretical analysis of SLICE. We make some practical assumptions. We first assume that DDIM inversion is accurate on most locations of each region.

\begin{assumption}[Bounded DDIM inversion error]\label{as:bound_err}
    Assume that for each factor $k \in \cK$, there exists a set $A_k \subseteq \Omega_k$ such that 
    \begin{align*}
        |A_k| \geq (1-\beta_k)|\Omega_k|, ~~~\mathrm{and}~~~ \|\bz_\mathrm{inv}(p) - \bz_T(p)\|_2 \leq \epsilon_k, ~\forall p \in A_k.
    \end{align*}
\end{assumption}

We make the following assumption to model the distinct behaviors of untampered and tampered semantic factors after semantic re-extraction.

\begin{assumption}[Stability and separation of semantic corruption]\label{as:sem_pertb}
    Let $\cJ$ $\subseteq \cK$ be the set of tampered semantic factors. Assume that for each untampered factor $k \in K \setminus \cJ$, there exists a set $B_k \subseteq \Omega_k$ such that \begin{align*}
        |B_k| \geq (1-\gamma_k)|\Omega_k|, ~~~\mathrm{and}~~~ \|\wh{\bz}_T - \bz_T(p)\|_2 \leq \delta_k, ~\forall p \in B_k,
    \end{align*}
    and for each tampered factor $k \in \cJ$,  there exists a set $C_k \subseteq \Omega_k$ such that
    \begin{align*}
        |C_k| \geq \rho_k|\Omega_k|, ~~~\mathrm{and}~~~ \|\wh{\bz}_T - \bz_T(p)\|_2 \geq \Delta_k, ~\forall p \in C_k.
    \end{align*}
\end{assumption}

Under Assumptions~\ref{as:bound_err} and~\ref{as:sem_pertb}, we analyze the partition-wise match ratios when only a subset of semantic factors is corrupted. The key is to choose the local thresholds so that untampered factors tolerate the combined inversion and re-extraction errors, while tampered factors remain separated by a clear margin. The following theorem quantifies the resulting localization guarantees.

\begin{theorem}[Robust localization under partial corruption]\label{thm:main}
    Let $\cJ \subseteq \cK$ be the set of tampered semantic factors. Assume that Assumptions~\ref{as:bound_err} and~\ref{as:sem_pertb} hold. If the set of local threshold $\{\tau_k\}_{k\in\cK}$ satisfies $\tau_k \geq \epsilon_k + \delta_k$ for all $k \in \cK \setminus\cJ$ and $\tau_k < \Delta_k - \epsilon_k$ for all $k \in J$, then the following holds:
    \begin{enumerate}
        \item For every untampered factor $k \in \cK \setminus \cJ$, we have $m_k \geq 1 - \beta_k  - \gamma_k$.
        \item For every tampered factor $k \in \cJ$, we have $m_k \leq 1 - (\rho_k - \beta_k)_+$.
        \item The global match ratio satisfies
        \begin{align*}
            \sum_{k \in \cK \setminus \cJ} \frac{|\Omega_k|}{hw}(1-\beta_k - \gamma_k) \leq m_g \leq \sum_{k \in\cK \setminus \cJ } \frac{|\Omega_k|}{hw} + \sum_{k \in \cJ} \frac{|\Omega_k|}{hw} \Big( 1 - (\rho_k - \beta_k)_+ \Big).
        \end{align*}
    \end{enumerate}
    We write $a_+ = \max\{a, 0\}$ for any $a \in \R$.
\end{theorem}

Theorem~\ref{thm:main} formalizes the core intuition behind SLICE under partial semantic corruption. With the above threshold choice, untampered regions retain a high match ratio whereas any tampered region must lose matches on a nontrivial subset of locations, and the global match ratio is bounded by the combined contributions of untampered and tampered regions. This explains why SLICE can preserve global provenance evidence while localizing the likely tampered semantic factor.

Next, we provide a statistical security guarantee against keyless or unwatermarked inputs. 

\begin{theorem}[Exponential false-accept bound for keyless or unwatermarked inputs]\label{thm:exp}
    Let $x_s$ be an unwatermarked image or keyless forgery attempt. We define $I_p = \mathbf{1}[d(p) \leq \tau_k].$ Assume that conditioned on the re-extracted semantics $\wt \cS$, the variables $\{I_p\}_p$ are independent and for all position $p$, $\P(I_p = 1 \mid\wt \cS) \leq q$ for some $q > 0$. Then the following holds:
    \begin{enumerate}
        \item If $q < \tau_g$, the probability that the image is accepted as ``watermark present'' (i.e. State I or State II) statisfies
        \begin{align*}
            \P(\mathrm{State~I~or~State~II}  \mid \wt \cS) \leq \exp\Big(-hw D_\mathrm{KL}(\tau_g \| q)\Big),
        \end{align*}
        where $ D_\mathrm{KL}(a \| b) =  a \ln\frac{a}{b} + (1-a) \ln\frac{1-a}{1-b}$ is the KL-divergence.
        \item If $q < \tau_k$ for every $k \in \cK$, then the probability that the image is accepted as authentic and intact (i.e. State I) satisfies
        \begin{align*}
            \P( \mathrm{State~I} \mid \wt \cS) \leq \exp\left( -\sum_{k \in \cK} |\Omega_k| D_\mathrm{KL}(\tau_k \| q)\right).
        \end{align*}
    \end{enumerate}
\end{theorem}

Under the assumption that accidental pointwise matches occur independently and with probability at most $q$, Theorem~\ref{thm:exp} guarantees that the probability that such an input exceeds the global or local verification thresholds decays exponentially with the number of latent positions.
\section{Prompt Language Selection}\label{sec:analysis}

\begin{wrapfigure}{r}{0.52\textwidth}
  \centering
  \vspace{-2mm}
  \includegraphics[width=\linewidth]{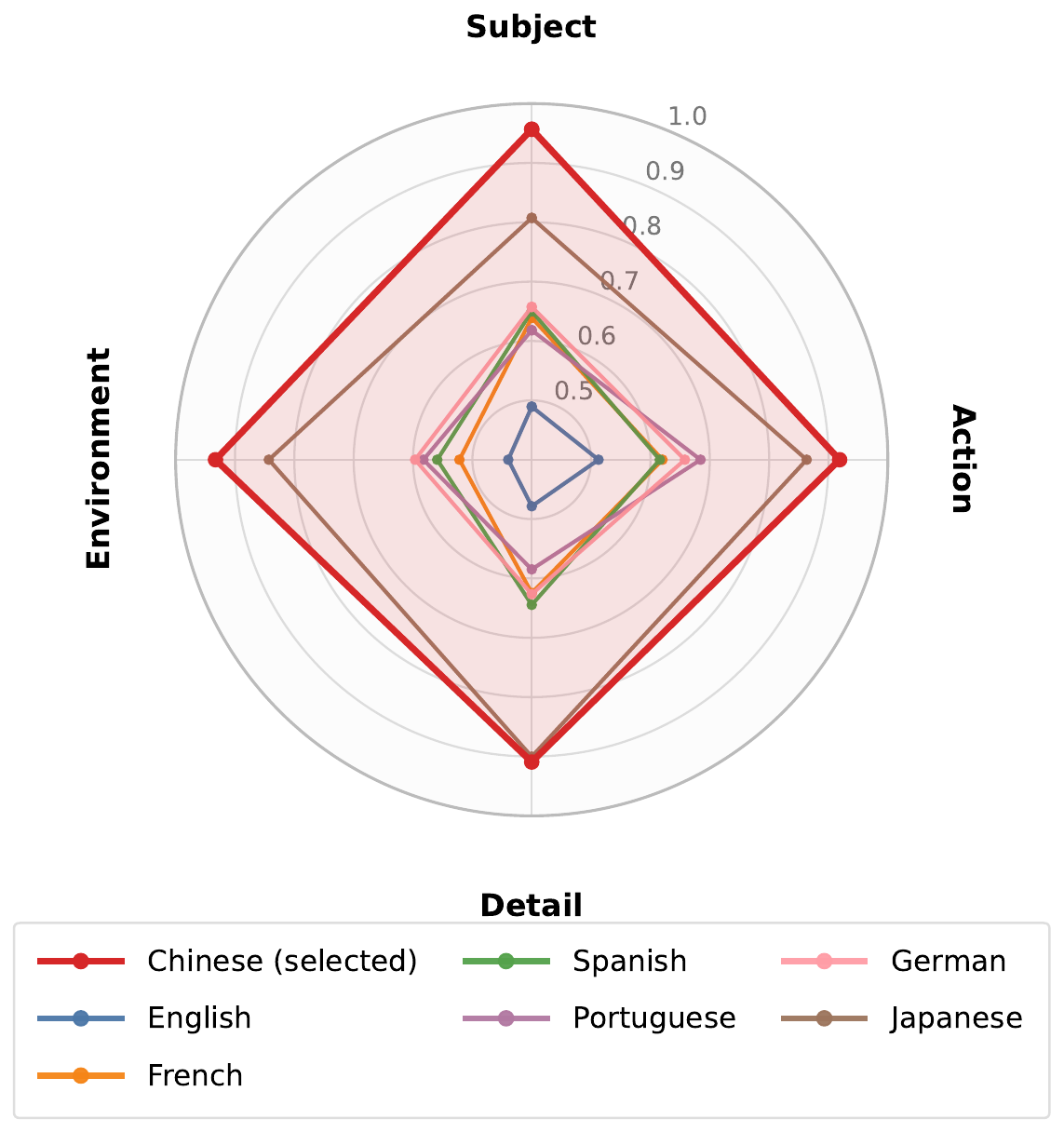} 
  \caption{\textbf{Semantic extraction stability across prompt languages}. The axes represent text embedding cosine similarity between initial and re-extracted descriptors.}
  \label{fig:radar_language}
  \vspace{-7mm}
\end{wrapfigure}

Reliable semantic extraction is critical for SLICE, as watermark verification requires the semantic descriptors extracted from the watermarked image to remain consistent with those from the reference image. However, Vision–Language Models (VLMs) may exhibit different instruction-following behaviors across languages, potentially introducing semantic drift during extraction.

To study this effect, we evaluate seven prompt languages (Chinese, English, French, Spanish, Portuguese, German, and Japanese). Our evaluation follows a ``Semantic Consistency Loop''. First, we deploy a VLM to extract the initial semantic descriptor set $S_\mathrm{ref}$ from the original reference image $\bx_\mathrm{ref}$. Second, using $S_\mathrm{ref}$ as the anchor, we synthesize the corresponding noise partitions and generate the watermarked image $\bx_w$. Third, we re-extract the semantic descriptors $S_w$ from $\bx_w$ using the same VLM and Meta-Prompt configuration. Finally, we utilize a BLIP encoder to extract the text embeddings of $S_\mathrm{ref}$ and $S_w$, and quantify the extraction stability by calculating their cosine similarity.

As shown in Fig.~\ref{fig:radar_language}, Chinese prompts consistently achieve the highest stability across all four semantic dimensions (\textit{subject}, \textit{environment}, \textit{action}, and \textit{detail}), with cosine similarity values approaching 1.0. In contrast, other languages show noticeable semantic drift, particularly in ``\textit{detail}'' and ``\textit{action}'' dimensions. We attribute this behavior to stronger alignment between the Qwen3-VL model and Chinese prompts, which improves the stability of cross-modal semantic representations. Based on this observation, we adopt Chinese as the default prompt language in our semantic extraction module. 
\section{Experiment}\label{sec:experi}
In this section, we evaluate the effectiveness and robustness of our proposed watermark. We employ the Stable Diffusion V2 model as the primary generation model and utilize QWen3-VL as the Vision-Language Model (VLM) for semantic extraction.

\textbf{Watermark baselines}: We select four state-of-the-art watermark techniques for comparison, including content-independent watermarks: Gaussian Shading Watermark (GSW) , Tree-Ring Watermark (TRW) , WIND , and the content-aware semantic watermark SEAL.

\textbf{Attack baselines}: To thoroughly test the robustness of our defense, we subject the watermarked images to three advanced semantic attacks: Latent Forgery Attack (LFA)~\cite{jain2025forging-article} , Regeneration with the Private Model (RPM)~\cite{Arabi25SEAL-long}, and Coherence-Preserving Semantic Injection (CSI)~\cite{gao2026breakingsemanticawarewatermarksllmguided}.

\definecolor{MyGreen}{RGB}{180, 232, 200}
\definecolor{HeaderGray}{gray}{0.9}

\begin{table}[t]
\centering
\small
\begin{threeparttable}
\caption{\textbf{Quantitative comparison of watermark resistance against generative semantic attacks}.}
\label{tab:detector_comparison}
\begin{tabular}{l|c c|c}
\hline
\rowcolor{HeaderGray}
\textbf{Detector (ASR \% $\downarrow$)} & \textbf{LFA~\cite{jain2025forging-article}} & \textbf{RPM~\cite{Arabi25SEAL-long}} & \textbf{CSI~\cite{gao2026breakingsemanticawarewatermarksllmguided}}\\
\hline
\multicolumn{4}{c}{\textit{Content-independent semantic watermarks}} \\
\hline
Gaussian Shading~\cite{Yang24GaussianShading-long} & 100 & 100 & \textbf{100}  \\
Tree-Ring~\cite{wen2023treering-article} & 93.81 & 100 & \textbf{100}  \\
WIND~\cite{arabi2025hidden} & 100 & 100 & \textbf{100} \\
\hline
\multicolumn{4}{c}{\textit{Content-aware semantic watermark}} \\
\hline
SEAL~\cite{Arabi25SEAL-long} & 0 & 7 & \textbf{81}\\
\rowcolor{MyGreen!55}
SLICE & 0 & 5 & \textbf{19}\\
\hline
\textbf{Mean $\downarrow$} & 58.76 & 62.40 & \textbf{80.00} \\
\hline
\end{tabular}
\end{threeparttable}
\end{table}

\subsection{Resistance Against Generative Forgery Attacks}
To evaluate the resistance of our proposed watermarking mechanism against generative forgery attacks, we conduct a series of comparative experiments. In this setup, we comprehensively compare the defensive capabilities of our method against mainstream semantic watermarks (Gaussian Shading, Tree-Ring, and WIND) and a content-aware semantic watermark (SEAL). We continue to employ the Attack Success Rate (ASR) as the core evaluation metric. From a defensive perspective, a successful attack is defined as a tampered image erroneously triggering a positive "watermark present" decision from the detector under the predefined threshold. Consequently, a lower ASR indicates stronger resistance of the watermarking scheme against forgery attacks. 

As shown in Table~\ref{tab:detector_comparison} , when facing generative semantic attacks including LFA, RPM, and the highly potent CSI, the defensive lines of the semantic watermarks (Gaussian Shading, Tree-Ring, and WIND) almost entirely collapse, with the ASRs reaching nearly 100\% across the board. In the comparison of content-aware semantic watermarks, although SEAL effectively defends against LFA (0\% ASR) and RPM (7\% ASR), its defense is severely breached by the more advanced CSI attack, with the ASR surging to 81\%. In contrast, our proposed method records exceptionally low ASRs of 0\% and 5\% against LFA and RPM, respectively. More importantly, even when subjected to the strongest CSI attack that compromised SEAL, our method successfully suppresses the ASR to a remarkably low 19\%.

These results demonstrate that our mechanism exhibits superior resistance compared to existing baseline defenses. This enhancement is primarily attributed to our Spatially-Partitioned Semantic Injection mechanism. By strictly binding spatial regions to fine-grained semantic attributes within the latent space , any attempt by an attacker to modify local semantics (e.g., maliciously replacing the "Environment" or "Action") immediately disrupts the localized watermark structure within the corresponding partition. Coupled with our rigorous multi-granularity verification logic , the detector can acutely capture localized tampering and trigger an alert, thereby rendering advanced generative forgery attacks ineffective.

\begin{table}[t]
\centering
\caption{\textbf{Resistance of SLICE against common image perturbations.}} 
\label{tab:robustness_no_cs}
\setlength{\tabcolsep}{8pt}
\begin{tabular}{l c c c c c c}
\toprule
\textbf{Transformation} & Clean & Rotate & JPEG & Blur & Noise & Bright \\
\midrule
\textbf{Accuracy}       & \textbf{1.000} & \textbf{1.000} & 0.990 & 0.988 & 0.993 & 0.941 \\
\bottomrule
\end{tabular}
\end{table}

\begin{table}[!t]
\centering
\caption{\textbf{CLIP scores before and after SLICE watermark injection}.}
\label{tab:clip_scores}
\setlength{\tabcolsep}{3.5mm}
\begin{tabular}{@{} cccc @{}}
\toprule
\multicolumn{2}{c}{\textbf{SDP}} & \multicolumn{2}{c}{\textbf{COCO}} \\
\cmidrule(lr){1-2} \cmidrule(lr){3-4}
CLIP (before) & CLIP (after) & CLIP (before) & CLIP (after) \\
\midrule
33.034 & 32.789 & 31.342 & 31.240 \\
\bottomrule
\end{tabular}
\end{table}

\subsection{Robustness Against Common Image Corruptions}

To evaluate the reliability of the proposed watermarking mechanism in practical dissemination, we test its robustness against five common non-malicious image perturbations (rotation, JPEG compression, Gaussian blur, Gaussian noise, and brightness adjustment) using detection accuracy as the evaluation metric. As shown in Table~\ref{tab:robustness_no_cs}, our method maintains a perfect accuracy of 1.0 under both the unperturbed (Clean) baseline and the rotation operation. When subjected to JPEG compression, blur, and noise interference, the accuracy is scarcely affected, reaching 0.990, 0.988, and 0.993, respectively. Even under brightness adjustment, which causes global shifts in pixel values, it sustains a high level of 0.941, demonstrating exceptional stability.

This outstanding anti-interference capability is directly attributed to our design that decouples the watermark from fragile low-level pixels. Unlike traditional methods sensitive to pixel shifts, our verification logic relies on the core semantic features extracted by the VLM. Since the aforementioned common image degradations typically cannot substantially alter the macroscopic semantics of an image, the locally reconstructed reference noise remains highly stable. This demonstrates that our high-level semantic watermark possesses remarkable robustness in real-world image propagation scenarios.

\subsection{Evaluation of Semantic Alignment and Visual Fidelity}

To evaluate the impact of SLICE on the generated image quality and semantic fidelity, we measure the text-image alignment before and after watermark injection across the Stable-Diffusion-Prompts (SDP)~\cite{Santana2024StableDiffusionPrompts} and COCO datasets~\cite{lin2014microsoft}, as shown in Table~\ref{tab:clip_scores}. Using the CLIP score~\cite{hessel2021clipscore} as our quantitative evaluation metric, the results demonstrate that our method introduces a negligible impact: on the SDP dataset, the average CLIP score experiences a marginal decrease from 33.034 to 32.789, with a similarly slight variation observed on the COCO dataset (from 31.342 to 31.240).

Beyond quantitative metrics, preserving the perceptual quality and aesthetic integrity of the generated images is a critical prerequisite for any practical watermarking system. To further demonstrate the stealthiness of our approach, we provide a comprehensive qualitative comparison between the unwatermarked baseline generations and our SLICE-protected images in Fig.~\ref{fig:qualitative}. As observed in the visual results across both datasets, SLICE introduces no perceptible artifacts, color shifts, or structural distortions. 

\begin{figure}[!t]
  \centering
  \includegraphics[width=\linewidth]{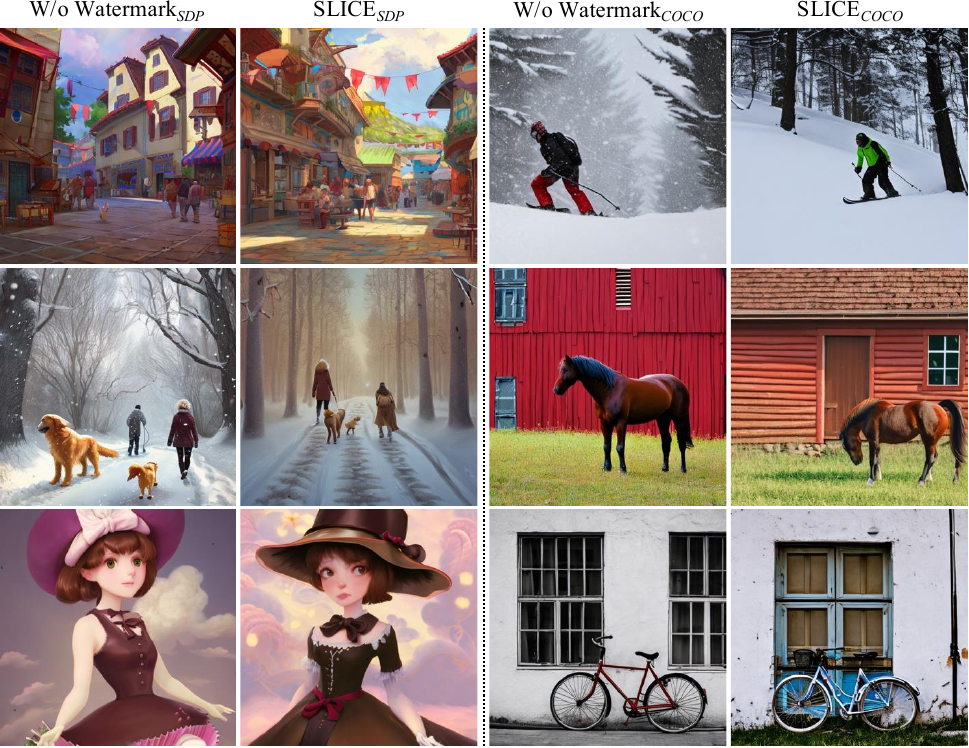}
  \caption{\textbf{Qualitative comparison of visual fidelity with and without SLICE watermarking.}} 
  \label{fig:qualitative}
\end{figure}

This exceptional semantic preservation and high visual fidelity are direct consequences of our Spatially-Partitioned Semantic Injection mechanism. Rather than dispersing cryptographic signals globally—which often leads to chaotic feature bleeding and disrupts the global latent space—our approach strictly restricts the decoupled semantic embeddings to non-overlapping local noise partitions. Consequently, SLICE achieves highly robust, multi-granularity provenance verification while flawlessly maintaining the original visual perceptual quality and consistency with the user's text prompts.

\begin{figure}[tbp]
    \centering
    \includegraphics[width=\textwidth]{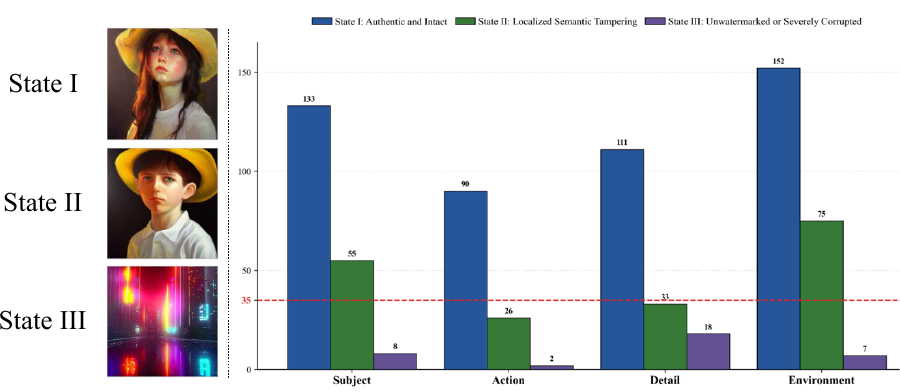} 
    \caption{\textbf{Case study of the proposed multi-granularity verification mechanism}.}
    \label{fig:verification_case}
\end{figure}
\subsection{Case Study: Multi-Granularity Verification}

To intuitively demonstrate SLICE's multi-granularity verification, we evaluate a set of images using our three-state detection mechanism, as shown in Fig. \ref{fig:verification_case}. The local match threshold for all semantic partitions is empirically set to $\tau_k=35$.

\textbf{State I (Authentic \& Intact):} For the original watermarked image (a young girl in a yellow hat), the matched patches across all four semantic factors—\textit{subject} (133), \textit{action} (90), \textit{detail} (111), and \textit{environment} (152)—significantly exceed the threshold. This robust alignment verifies both global provenance and semantic integrity.

\textbf{State II (Localized Tampering):} When the image undergoes localized tampering (e.g., altered from a girl to a young boy staring forward), our compartmentalized embedding strictly confines the disruption to specific semantic regions. As illustrated in the middle panel of Fig. \ref{fig:verification_case}, changing the character's gender causes the \textit{subject} matches to drop significantly to 55. Simultaneously, the altered behavior ("staring forward") severely destroys the \textit{action} partition (26) and its associated \textit{detail} (33), causing both to fall below the verification threshold. However, the \textit{environment} (75) and \textit{subject} (55) remain above the threshold, correctly triggering State II. This visually demonstrates SLICE's capability to precisely capture and localize fine-grained semantic modifications without outright rejecting an image that retains partial provenance.

\textbf{State III (Unwatermarked / Severely Corrupted):} For a completely unrelated image (a cyberpunk nightscape), extreme semantic mismatch causes the matched patches for all factors—\textit{subject} (8), \textit{action} (2), \textit{detail} (18), and \textit{environment} (7)—to plummet far below the threshold, correctly triggering the unwatermarked classification.

Overall, this case study visually validates that SLICE's spatial binding elegantly translates complex semantic edits into readable, partition-specific verification states.

\subsection{Limitation and Mitigation: A Joint Defense Paradigm}

While SLICE demonstrates remarkable robustness against advanced semantic editing, its reliance on spatially-partitioned binding inherently makes it sensitive to extreme geometric distortions. As shown in Table~\ref{tab:mitigation_cs}, aggressive Cropping and Scaling (C\&S) attacks disrupt the strict spatial correspondence required for latent inversion, causing the active watermark detection AUC to drop to 0.054.

\begin{table}[h]
\centering
\caption{\textbf{Mitigation of C\&S attack via passive forensics.}}
\label{tab:mitigation_cs}
\resizebox{\linewidth}{!}{%
\renewcommand{\arraystretch}{1.15}
\begin{tabular}{@{}l c ccccc@{}}
\toprule
& \textbf{Vulnerability} & \multicolumn{5}{c}{\textbf{Mitigation via Passive Forensics (AUC $\uparrow$)}} \\
\cmidrule(lr){2-2} \cmidrule(l){3-7}
\textbf{Method} & Ours (C\&S) & Feature Stab. & Resampling & Freq-Texture & ResNet18 & Spectral Aniso. \\
\midrule
\textbf{AUC}  & 0.054 & 0.850 & 0.920 & 0.940 & 0.980 & \textbf{0.990} \\
\bottomrule
\end{tabular}
}
\end{table}

However, this specific vulnerability is readily mitigated by integrating passive forensic techniques~\cite{stamm2013information, farid2009image}. Extreme C\&S manipulations inevitably leave prominent unnatural artifacts, such as interpolation traces and spectral anomalies. Our evaluations reveal that standard passive detectors---spanning semantic CNNs~\cite{bayar2018constrained, zhou2018learning}, feature-stability metrics~\cite{choi2025trainingfree_warpad, choi2024hfi}, and spectral analysis~\cite{gabarda2007blind, popescu2005exposing}---can easily expose these brute-force attacks, achieving high detection AUCs (0.85--0.99). Therefore, combining the active SLICE watermark with passive forensics establishes a highly robust joint defense paradigm, effectively neutralizing attacks that attempt to bypass semantic verification through destructive spatial transformations.

\section{Conclusion}\label{sec:conclusion}

In this paper, we identify a critical vulnerability in global semantic watermarks: their susceptibility to locally fine-grained but globally coherent edits. To overcome this, we propose Semantic Latent Injection via Compartmentalized Embedding (SLICE). By decoupling image semantics into four factors and anchoring them to spatially-partitioned latent regions, SLICE establishes a powerful multi-granularity verification logic. Extensive experiments demonstrate that our training-free framework significantly outperforms baselines against advanced generative attacks (e.g., LFA, RPM, CSI), accurately localizes tampering, and preserves exceptional visual fidelity. 

While our spatial binding remains sensitive to extreme geometric distortions, integrating passive forensics effectively mitigates this through a robust joint defense. Moving forward, designing spatially invariant decoupled semantic embeddings remains a promising direction to further advance AI provenance techniques.

\bibliographystyle{splncs04}
\bibliography{main}

\clearpage 

\appendix

\begin{center}
    \textbf{\huge Appendix}
\end{center}
\section{Proof of Theorem~\ref{thm:main}}

\begin{theorem}[Restatement of Theorem~\ref{thm:main}]
    Let $\cJ \subseteq \cK$ be the set of tampered semantic factors. Assume that Assumptions~\ref{as:bound_err} and~\ref{as:sem_pertb} hold. If the set of local threshold $\{\tau_k\}_{k\in\cK}$ staisfies $\tau_k \geq \epsilon_k + \delta_k$ for all $k \in \cK \setminus\cJ$ and $\tau_k < \Delta_k - \epsilon_k$ for all $k \in J$, then the following holds:
    \begin{enumerate}
        \item for every untampered factor $k \in \cK \setminus \cJ$, we have $m_k \geq 1 - \beta_k  - \gamma_k$;
        \item for every tampered factor $k \in \cJ$, we have $m_k \leq 1 - (\rho_k - \beta_k)_+$;
        \item the global match ratio satisfies
        \begin{align*}
            \sum_{k \in \cK \setminus \cJ} \frac{|\Omega_k|}{hw}(1-\beta_k - \gamma_k) \leq m_g \leq \sum_{k \in\cK \setminus \cJ } \frac{|\Omega_k|}{hw} + \sum_{k \in \cJ} \frac{|\Omega_k|}{hw} \Big( 1 - (\rho_k - \beta_k)_+ \Big),
        \end{align*}
    \end{enumerate}
    where $a_+ = \max\{a, 0\}$ for any $a \in \R$.
\end{theorem}
\begin{proof}
    Let $k \in \cK \setminus \cJ$ be an untampered factor. For any $p \in A_k \cap B_k$, we can show that
    \begin{align*}
        d(p) = \|\bz_\mathrm{inv}(p) - \wh \bz(p) \|_2 \leq \|\bz_\mathrm{inv}(p) - \bz_T(p) \|_2 + \|\bz_T(p) - \bz_T(p) \|_2
        \leq \epsilon_k + \delta_k \leq \tau_k,
    \end{align*}
    where the second step uses the triangle inequality. Hence every point in $A_k \cap B_k$ is counted as a match. Now, by inclusion-exclusion principle, we have
    \begin{align*}
        |A_k \cap B_k| \geq |A_k| + |B_k| - |\Omega_k| \geq &~ (1-\beta_k)|\Omega_k| + (1-\gamma_k)|\Omega_k| - |\Omega_k| \\ = &~ (1-\beta_k - \gamma_k) |\Omega_k|.
    \end{align*}
    Thus
    \begin{align*}
        m_k = \frac{1}{|\Omega_k|} \sum_{p \in \Omega_k} \mathbf{1}[d(p) \leq \tau_d] \geq \frac{|A_k \cap B_k|}{|\Omega_k|} \geq 1 - \beta_k - \gamma_k.
    \end{align*}
    This proves the first part.

    Let $k \in \cJ$ be a tampered factor. For any $p \in A_k \cap C_k$, we can show that
    \begin{align*}
         d(p) = \|\bz_\mathrm{inv}(p) - \wh \bz(p) \|_2 \geq &~ \left| \|\|\bz_T(p) - \wh\bz_T(p) \|_2 - \bz_\mathrm{inv}(p) - \bz_T(p) \|_2  \right|
        \\ \geq &~ \Delta_k - \epsilon_k > \tau_k,
    \end{align*}
    where the second step uses the reverse triangle inequality. Hence, every point in $A_k \cap C_k$ is a non-match. Again by inclusion-exclusion principle, we have
    \begin{align*}
        |A_k \cap C_k| \geq |A_k| + |C_k| - |\Omega_k| \geq (1-\beta)|\Omega_k|+ \rho_k|\Omega_k| - |\Omega_k| = (\rho_k - \beta_k)|\Omega_k|.
    \end{align*}
    Note that when $\rho_k - \beta_k < 0$, the lower bound trivially holds, so we have
    \begin{align*}
        |A_k \cap C_k| \geq  (\rho_k - \beta_k)_+ |\Omega_k|.
    \end{align*}
    Note that since all points in $A_k \cap C_k$ fail to match, the number of matches in $\Omega_k$ is at most
    \begin{align*}
        |\Omega_k| - |A_k \cap C_k| \leq  \Big( 1 - (\rho_k - \beta_k)_+ \Big) |\Omega_k|.
    \end{align*}
    Dividing by $|\Omega_k|$ gives
    \begin{align*}
        m_k \leq 1 - (\rho_k - \beta_k)_+.
    \end{align*}
    Thus we prove the second part.

    For the global lower bound, all matches in untampered partitions contribute a nonnegative amount, and by the first part, we have
    \begin{align*}
        m_g = \frac{1}{hw} \sum_{k \in \cK}\sum_{p \in \Omega_k}\mathbf{1}[d(p) \leq \tau_k] \geq \frac{1}{hw} \sum_{k \in \cK \setminus \cJ}(1-\beta_k - \gamma_k)|\Omega_k|.
    \end{align*}

    For the global upper bound, an untampered region can contribute at most all of its locations, while a tampered region contributes at most $( 1 - (\rho_k - \beta_k)_+) |\Omega_k|$ by the second part. Hence
    \begin{align*}
        m_g \leq \frac{1}{hw}\sum_{k \in \cK \setminus \cJ}|\Omega_k| + \frac{1}{hw}\sum_{k \in \cJ}\Big( 1 - (\rho_k - \beta_k)_+ \Big) |\Omega_k|.
    \end{align*}
    This proves the third part.
\end{proof}

\section{Proof of Theorem~\ref{thm:exp}}



\begin{lem}[Chernoff bounds, Theorem 2.17 in \cite{zhang2023mathematical}]\label{lem:chernoff}
    Let $X_1,\ldots,X_n$ be $n$ independent random variables taking values in $[0,1]$. Let $\mu = \E[\frac{1}{n}\sum_{i=1}^n X_i]$. Then for all $\epsilon > 0$, we have
    \begin{align*}
        \P\left(\frac{1}{n}\sum_{i=1}^n X_i \geq \mu + \epsilon\right) \leq \exp\Big(-n D_\mathrm{KL}(\mu + \epsilon \| \mu)\Big), \\
        \P\left(\frac{1}{n}\sum_{i=1}^n X_i \leq \mu - \epsilon\right) \leq \exp\Big(-n D_\mathrm{KL}(\mu - \epsilon \| \mu)\Big).
    \end{align*}
    where $D_\mathrm{KL}(a \| b) =  a \ln\frac{a}{b} + (1-a) \ln\frac{1-a}{1-b}$ is the KL divergence.
\end{lem}

\begin{theorem}[Restatement of Theorem~\ref{thm:exp}]
    Let $x_s$ be an unwatermarked image or keyless forgery attempt. We define $I_p = \mathbf{1}[d(p) \leq \tau_k].$ Assume that condition on the re-extracted semantics $\wt \cS$, the variables $\{I_p\}_p$ are independent and for all position $p$, $\P(I_p = 1 | \wt \cS) \leq q$ for some $q > 0$. Then
    \begin{enumerate}
        \item If $q < \tau_g$, the probability that the image is accepted as ``watermark present'' (i.e. State I or State II) satisfies
        \begin{align*}
            \P(\mathrm{State~I~or~State~II}  \mid \wt \cS) \leq \exp\Big(-hw D_\mathrm{KL}(\tau_g \| q)\Big).
        \end{align*}
        \item If $q < \tau_k$ for every $k \in \cK$, then 
        \begin{align*}
            \P( \mathrm{State~I} \mid \wt \cS) \leq \exp\left( -\sum_{k \in \cK} |\Omega_k| D_\mathrm{KL}(\tau_k \| q)\right).
        \end{align*}
    \end{enumerate}
\end{theorem}

\begin{proof}
Conditioned on $\wt{\cS}$ throughout. 
For the first part, let $S = \sum_{p} I_p$.
Then $m_g = \frac{S}{hw}.$
By definition, the event that the image is accepted as ``watermark present'' (i.e. $\mathrm{State~I}$ or $\mathrm{State~II}$) is exactly the event
\begin{align*}
    \{m_{\mathrm{g}} \geq \tau_g\} = \{S \geq \tau_g hw\}.
\end{align*}
Applying the Chernoff bound (Lemma~\ref{lem:chernoff}), we obtain
\begin{align*}
    \mathbb{P}(m_g \geq \tau_g \mid \wt{\cS})
\leq \exp\Big(-hw D(\tau_g\|q)\Big).
\end{align*}
This proves the first part.

For the second part, define for each partition
$
S_k = \sum_{p \in \Omega_k} I_p.
$
Then $m_k = \frac{S_k}{|\Omega_k|}.$
If $\mathrm{State~I}$ occurs, then $m_k \geq \tau_k$ for every $k \in \cK$. Therefore,
\begin{align*}
    \{\mathrm{State~I}\} \subseteq \bigcap_{k \in \cK} \{S_k \geq \tau_k |\Omega_k|\}.
\end{align*}
Hence
\begin{align*}
    \mathbb{P}(\mathrm{State~I} \mid \wt{\cS})
\leq
\mathbb{P}\left( \bigcap_{k \in \cK} \{S_k \geq \tau_k |\Omega_k|\} \,\middle|\, \wt{\cS} \right).
\end{align*}
Since the partitions $\{\Omega_k\}_{k\in\cK}$ are disjoint and the variables $\{I_p\}_p$ are conditionally independent given $\wt{\cS}$, the random sums $\{S_k\}_{k\in\cK}$ are conditionally independent. Thus,
\begin{align*}
\mathbb{P}(\mathrm{State~I} \mid \wt{\cS})
&\leq \prod_{k \in \cK} \mathbb{P}(S_k \geq \tau_k |\Omega_k| \mid \wt{\cS}).
\end{align*}
Applying the Chernoff bound (Lemma~\ref{lem:chernoff}) again, we have
\begin{align*}
    \mathbb{P}(S_k \geq \tau_k |\Omega_k| \mid \wt{\cS})
\leq \exp(-|\Omega_k| D(\tau_k\|q)).
\end{align*}
Then using the union bound over $k$, we have
\begin{align*}
    \mathbb{P}(\mathrm{State~I} \mid \wt{\cS})
\leq
\exp\left( - \sum_{k\in\cK} |\Omega_k| D(\tau_k\|q) \right).
\end{align*}
This proves the second part.
\end{proof}

\end{document}